\def\year{2017}\relax
\documentclass[letterpaper]{article}
\usepackage{aaai17}
\usepackage[utf8]{inputenc} 
\usepackage{times}
\usepackage{helvet}
\usepackage{courier}
\usepackage{url}            
\usepackage{booktabs}       
\usepackage{amsfonts}       
\usepackage{nicefrac}       
\usepackage{microtype}      
\usepackage{float}
\usepackage[english]{babel}
\usepackage{csquotes}
\usepackage{graphicx}

\clearpage{}

\DeclareSymbolFont{bbold}{U}{bbold}{m}{n}
\DeclareSymbolFontAlphabet{\mathbbold}{bbold}
\usepackage{enumerate}
\usepackage{amsmath}
\usepackage{amsfonts}
\usepackage{amssymb}
\usepackage{amsbsy}
\usepackage{bbm}
\usepackage{isomath}
\usepackage{amsthm}
\usepackage{dsfont}
\usepackage{algorithm}
\usepackage{algpseudocode}
\usepackage{mathrsfs}
\usepackage{paralist}
\usepackage{epsfig}
\usepackage[caption=false]{subfig}
\usepackage{makeidx} 
\usepackage{mathtools}
\usepackage{comment}
\usepackage{breqn}
\usepackage{pgf}
\usepackage{tikz}

\numberwithin{equation}{section}

\theoremstyle{plain}
\newtheorem{corollary}{Corollary}

\newtheorem{theorem}{Theorem}

\theoremstyle{definition}
\newtheorem{definition}{Definition}

\theoremstyle{remark}

\numberwithin{corollary}{section}
\numberwithin{lemma}{section}
\numberwithin{theorem}{section}
\numberwithin{assumption}{section}
\numberwithin{fact}{section}
\numberwithin{definition}{section}
\numberwithin{example}{section}
\numberwithin{conjecture}{section}
\numberwithin{remark}{section}
\numberwithin{claim}{section}

\newcommand \E {\mathop{\mbox{\ensuremath{\mathbb{E}}}}\nolimits}

\renewcommand \Pr {\mathop{\mbox{\ensuremath{\mathbb{P}}}}\nolimits}

\newcommand \CA {{\mathcal{A}}}

\DeclareMathAlphabet{\mathpzc}{OT1}{pzc}{m}{it}

\newcommand \Bernoulli {\mathop{\mathpzc{Bern}}\nolimits}

\newcommand \Laplace {\mathop{\mathpzc{Lap}}\nolimits}

\newcommand \Actions {\CA}

\renewcommand{\O}[1]{\mathcal{O}(#1)}

\DeclarePairedDelimiter{\abs}{\lvert}{\rvert}

\usepackage{tikz}
\usetikzlibrary{automata,topaths,shapes,arrows,fit,decorations.markings}
\tikzstyle{utility}=[diamond,draw=black,draw=blue!50,fill=blue!10,inner sep=0mm, minimum size=8mm]
\tikzstyle{select}=[rectangle,draw=black,draw=blue!50,fill=blue!10,inner sep=0mm, minimum size=6mm]
\tikzstyle{hidden}=[dashed,draw=black]
\tikzstyle{RV}=[circle,draw=black,draw=blue!50,fill=blue!10,inner sep=0mm, minimum size=6mm]

\def\clap#1{\hbox to 0pt{\hss#1\hss}}

\makeatletter
\let\OldStatex\Statex
\renewcommand{\Statex}[1][3]{
	\setlength\@tempdima{\algorithmicindent}
	\OldStatex\hskip\dimexpr#1\@tempdima\relax}
\makeatother

 \renewcommand\abs[1]{\left|#1\right|}
\clearpage{}

\newcommand{\alg}{\ensuremath{\Lambda}}
\newcommand{\Acts}{\mathcal{A}}
\newcommand{\Regret}{\mathcal{R}}
\newcommand{\AEXP}{{\emph{EXP3}}}

\newcommand{\DPEXPLAP}{{\emph{DP-EXP3-Lap}}}
\newcommand{\DPALAP}{{\ensuremath{\emph{DP-}\alg\emph{-Lap}}}}
\newcommand{\BASE}{{\ensuremath{{\alg}}}}

\newcommand{\Noise}{{\mathcal{N}}}

\newcommand{\EXPA}{{\textsc{EXP3}}}

\newcommand{\NTRIALS}{{\ensuremath{720}}}
\newcommand{\NGROUPS}{{\ensuremath{24}}}
\newcommand{\Id}{{\ensuremath{\mathds{1}}}}
\newcommand{\aGENT}{{agent}}
\newcommand{\aRM}{{arm}}
\newcommand{\aGENTS}{{agents}}
\newcommand{\aRMS}{{arms}}

\newcommand{\rOUND}{{round}}
\newcommand{\rOUNDS}{{rounds}}
\newcommand{\gAIN}{{gain}}
\newcommand{\gAINS}{{gains}}

\newcommand{\gain}{{g}}
\newtheorem{proposition}[theorem]{Proposition}

\begin{document}
	\title{Achieving Privacy in the Adversarial Multi-Armed Bandit}
 \author{Aristide C. Y. Tossou \\ Chalmers University of Technology \\ Gothenburg, Sweden\\ aristide@chalmers.se
 \And
Christos Dimitrakakis\\
University of Lille, France\\
Chalmers University of Technology, Sweden\\
Harvard University, USA\\christos.dimitrakakis@gmail.com}
	\maketitle
	
	\begin{abstract}
		In this paper, we improve the previously best known regret bound to achieve $\epsilon$-differential privacy in oblivious adversarial bandits from $\O{T^{2/3}/\epsilon}$ to $\O{\sqrt{T} \ln T /\epsilon}$. This is achieved by combining a Laplace Mechanism with \EXPA{}.
		We show that though \EXPA{} is already differentially private, it leaks a linear amount of information in $T$. However, we can improve this privacy by relying on its intrinsic exponential mechanism for selecting actions. 
		This allows us to reach $\O{\sqrt{\ln T}}$-DP, with a regret of $\O{T^{2/3}}$ that holds against an adaptive adversary, an improvement from the best known of $\O{T^{3/4}}$. This is done by using an algorithm that run \EXPA{} in a mini-batch loop. Finally, we run experiments that clearly demonstrate the validity of our theoretical analysis.
		
	\end{abstract}

	\section{Introduction}

We consider multi-armed bandit problems in the adversarial
setting whereby an
\aGENT{} selects one from a number of alternatives (called \aRMS{}) at
each \rOUND{} and receives a \gAIN{} that depends on its choice.  The
\aGENT{}'s goal is to maximize its total \gAIN{} over time. There are two main settings for the bandit problem. In the stochastic one, the \gAINS{} of each arm are
generated i.i.d by some unknown probability law. In the adversarial setting,
which is the focus of this paper, the \gAINS{} are generated adversarially.
We are interested in finding algorithms with a 
total gain over $T$ \rOUNDS{} not much smaller than that of an oracle with additional knowledge about the problem.
In both settings, algorithms that achieve the optimal (problem-independent) regret bound of $\O{\sqrt{T}}$  are
known~\cite{finitetimemab,burnetas1996optimal,handlingads,thompson1933lou,obliviousauer03,shiftband,agrawal:thompson}.

This problem is a model for many applications where there is a need for trading-off exploration and exploitation. This is so because, whenever we make a choice, we only observe the \gAIN{} generated by that choice, and not the \gAINS{} that we could have obtained otherwise. An example is clinical trials, where arms correspond to different treatments or tests, and the goal is to maximize the number of cured patients over time while being uncertain about the effects of treatments. Other problems, such as search engine advertisement and movie recommendations can be formalized similarly \cite{handlingads}.

Privacy can be a serious issue in the bandit setting (c.f. ~\cite{dponlinelearning,dplearningbanditandfull,mishra2015nearly,dpsmartgrid}). For example, in clinical trials, we may want to detect and publish results about the best drug without leaking sensitive information, such as the patient's health condition and genome.
\emph{Differential privacy}~\cite{dwork06dp} formally bounds the amount of information that a third party can learn no matter their power or side information.

Differential privacy has been used before in the stochastic setting~\cite{aaai16,mishra2015nearly,dponlinelearning} where the authors obtain optimal algorithms up to logarithmic factors.  In the adversarial setting, \cite{dplearningbanditandfull} adapts an algorithm called \emph{Follow The Approximate Leader} to make it private and obtain a regret bound of $\O{T^{2/3}}$.
In this
work, we show that a number of simple algorithms can satisfy privacy
guarantees, while achieving nearly optimal regret (up to logarithmic
factors) that scales naturally with the level of privacy desired.

Our work is also of independent interest for non-private multi-armed
bandit algorithms, as there are competitive with the
current state of the art against switching-cost adversaries (where we recover the optimal bound). Finally, we provide rigorous empirical results against a variety of adversaries.

The following section gives the main background and notations. Section \ref{sec:noise} describes meta-algorithms that perturb the gain sequence to achieve privacy, while Section~\ref{sec:lever-inher-priv} explains how to leverage the privacy inherent in the EXP3 algorithm by modifying the way gains are used. Section~\ref{sec:experiments} compares our algorithms with \EXPA{} in a variety of settings. The full proofs of all our main results are in the full version.

	\section{Preliminaries}

\subsection{The Multi-Armed Bandit problem}
Formally, a bandit game is defined between an adversary and an \aGENT{} as follows: there is a set of $K$ \aRMS{} $\Acts$, and at each \rOUND{} $t$, the \aGENT{} plays an \aRM{} $I_t \in \Acts$. Given the choice $I_t$, the adversary grants the \aGENT{} a \gAIN{} $g_{I_t,t} \in [0, 1]$.  The \aGENT{} only observes the \gAIN{} of \aRM{} $I_t$, and not that of any other \aRMS{}. The goal of this \aGENT{} is to maximize its total \gAIN{} after $T$ \rOUNDS{}, $\sum_{t=1}^T g_{I_t, t}$.
A randomized bandit algorithm $\alg : (\Actions \times [0,1])^{*} \to \mathscr{D}(\Actions)$ maps every \aRM-\gAIN{} history to a distribution over the next \aRM{} to take.

The nature of the adversary, and specifically, how the \gAINS{} are
generated, determines the nature of the game.  For the
\emph{stochastic} adversary ~\cite{thompson1933lou,finitetimemab},
the \gAIN{} obtained at \rOUND{} $t$ is generated i.i.d from a distribution
$P_{I_{t}}$. The more general \emph{fully oblivious} adversary
\cite{audibert10a} 
generates the \gAINS{} independently at \rOUND{} $t$ but not necessarily identically from a distribution $P_{I_{t}, t}$.
Finally, we have the \emph{oblivious} adversary~\cite{obliviousauer03}
whose only constraint is to generate the gain $g_{I_{t}, t}$ as a function of the current action $I_t$ only, i.e. ignoring previous actions and gains.

While focusing on oblivious adversaries, we discovered that by targeting differential privacy we can also compete against the stronger \emph{$m$-bounded memory adaptive adversary} \cite{switching,memoryadv,policyregret12} who can use up to the last $m$ \gAINS{}. The oblivious adversary is a special case with $m = 0$. Another special case of this adversary is the one with \emph{switching costs}, who penalises the agent whenever he switches arms, by giving the lowest possible gain of 0 (here $m = 1$).

\paragraph{Regret.}
Relying on the cumulative \gAIN{} of an \aGENT{} to evaluate its performance can be misleading. Indeed, consider the case where an adversary gives  a zero \gAIN{} for all \aRMS{} at every \rOUND{}. The cumulative \gAIN{} of the \aGENT{} would look bad but no other \aGENTS{} could have done better. This is why one compares the gap between the \aGENT{}'s cumulative \gAIN{} and the one obtained by some hypothetical  \aGENT{}, called \emph{oracle}, with additional information or computational power. This gap is called the \emph{regret}. 

There are also variants of the oracle that are considered in the literature. The most common variant is the \emph{fixed oracle}, which always plays the best fixed \aRM{} in hindsight. The regret $\Regret$  against this \emph{oracle}  is :
\[
\Regret = \max_{i = 1, \ldots K}\sum_{t=1}^{T} g_{i,t} - \sum_{t=1}^{T} g_{I_t, t}
\]
In practice, we either prove a high probability bound on $\Regret$ or an expected value $\E \Regret$ with:
\[
\E \Regret = \E \left[\max_{i = 1, \ldots K}\sum_{t=1}^{T} g_{i,t} - \sum_{t=1}^{T} g_{I_t, t} \right]
\]
where the expectation is taken with respect to the random choices of both the \aGENT{} and adversary. 
There are other oracles like the \emph{shifting oracle} but those are out of scope of this paper.

\paragraph{EXP3.} The Exponential-weight for Exploration and Exploitation (\EXPA{} \cite{obliviousauer03}) algorithm achieves the optimal bound (up to logarithmic factors) of $\O{\sqrt{TK\ln K}}$ for the weak regret (i.e. the expected regret compared to the \emph{fixed oracle}) against an \emph{oblivious adversary}. \EXPA{} simply maintains an estimate $\tilde{G}_{i,t}$ for the cumulative \gAIN{} of \aRM{} $i$ up to \rOUND{} $t$ with $\tilde{G}_{i,t} = \sum_{s=1}^{t} \frac{g_{i,t}}{p_{i,t}}\Id_{I_t=i}$ where 

\begin{equation}
\label{exp3:prob}
p_{i,t} = (1-\gamma)\frac{\exp{(\gamma /K\tilde{G}_{i,t})}}{\sum_{i=1}^{K} \exp{(\gamma /K\tilde{G}_{i,t})}} + \frac{\gamma}{K} 
\end{equation}

with $\gamma$ a well defined constant.

Finally, \EXPA{} plays one action randomly according to the probability distribution $p_t = \{p_{1,t}, \ldots p_{K,t} \}$ with $p_{i,t}$ as defined above.

\subsection{Differential Privacy}

 The following definition (from \cite{aaai16}) specifies what is meant when we called a bandit algorithm differentially private at a single \rOUND{} $t$:		
 \begin{definition}[Single \rOUND{} ($\epsilon, \delta)$-differentially private bandit algorithm]
 	A randomized bandit algorithm $\alg$ is $(\epsilon, \delta)$-differentially
 	private at \rOUND{} $t$, if for all sequence $\gain_{1:t-1}$ and $\gain'_{1:t-1}$ that
 	differs in at most one \rOUND{}, we have for any action subset $S
 	\subseteq \Acts$:
 	\begin{align}
 	\Pr_\alg(I_t \in S \mid \gain_{1:t-1})
 	\leq \delta + 
 	\Pr_\alg(I_t \in S \mid \gain'_{1:t-1})
 	e^\epsilon,
 	\label{eq:dp-bandit}
 	\end{align}
 	where
 	$\Pr_\alg$ denotes the probability distribution specified
 	by the algorithm and $g_{1:t-1} = \{g_1, \ldots g_{t-1} \}$ with $g_s$ the gains of all arms at \rOUND{} $s$.  When $\delta =
 	0$, the algorithm is said to be
 	\emph{$\epsilon$-differential private}.
 	\label{def:dp}
 \end{definition}
 
 The $\epsilon$ and $\delta$ parameters quantify the amount of privacy loss. Lower ($\epsilon$,$\delta$) indicate higher privacy and consequently we will also refer to ($\epsilon$,$\delta$) as the privacy loss. Definition \ref{def:dp} means that the output of the bandit algorithm at \rOUND{} $t$ is almost insensible to any single change in the \gAINS{} sequence. This implies that whether or not we remove a single \rOUND{}, replace the \gAINS{}, the bandit algorithm will still play almost the same action. Assuming the \gAINS{} at \rOUND{} $t$ are linked to a user private data (for example his cancer status or the advertisement he clicked), the definition preserves the privacy of that user against any third parties looking at the output. This is the case because the choices or the participation of that user would not almost affect the output. Equation \eqref{eq:dp-bandit} specifies how much the output is affected by a single user.
 
We would like Definition \ref{def:dp} to hold for all \rOUNDS{}, so as to protect the privacy of all users. If it does for some $(\epsilon, \delta)$, then we say the algorithm has \emph{per-\rOUND{}} or \emph{instantaneous} privacy loss $(\epsilon, \delta)$.  Such an algorithm also has a \emph{cumulative} privacy loss of at most $(\epsilon', \delta')$ with $\epsilon' = \epsilon T$ and $\delta' = \delta T$ after $T$ steps. Our goal is to design bandit algorithm such that their cumulative privacy loss $(\epsilon', \delta')$  are as low as possible while achieving simultaneously a very low regret. In practice, we would like $\epsilon'$ and the regret to be sub-linear while $\delta'$ should be a very small quantity. Definition \ref{def:dp:cumul} formalizes clearly the meaning of this cumulative privacy loss and for ease of presentation, we will ignore the term "cumulative" when referring to it.
		\begin{definition}[($\epsilon, \delta)$-differentially private bandit algorithm]
			A randomized bandit algorithm $\alg$ is $(\epsilon, \delta)$-differentially
			private up to \rOUND{} $t$, if for all  $\gain_{1:t-1}$ and $\gain'_{1:t-1}$ that
			differs in at most one \rOUND{}, we have for any action subset $S
			\subseteq \Acts^{t}$:
			\begin{align}
             \Pr_\alg(I_{1:t} \in S \mid \gain_{1:t-1})
					&\leq \delta + 
             \Pr_\alg(I_{1:t} \in S \mid \gain'_{1:t-1})
					e^\epsilon,
					\label{eq:dp-bandit:cumul}
			\end{align}
            where $\Pr_\alg$ and $\gain$ are as defined in Definition \ref{def:dp}.
			\label{def:dp:cumul}
		\end{definition}
		Most of the time, we will refer to Definition \ref{def:dp:cumul} and whenever we need to use Definition \ref{def:dp}, this will be made explicit.

The simplest mechanism to achieve differential privacy for a function is to add Laplace noise of scale proportional to its sensitivity. The sensitivity is the maximum amount by which the value of the function can change if we change a single element in the inputs sequence. For example, if the input is a stream of numbers in $[0, 1]$ and the function their sum, we can add Laplace noise of scale $\frac{1}{\epsilon}$ to each number and achieve $\epsilon$-differential privacy with an error of $\O{\sqrt{T}/\epsilon}$ in the sum. However, \cite{chan2010private} introduced \emph{Hybrid Mechanism}, which achieves $\epsilon$-differential privacy with only poly-logarithmic error (with respect to the true sum). The idea is to group the stream of numbers in a binary tree and only add a Laplace noise at the nodes of the tree.

As demonstrated above, the main challenge with differential privacy is thus to trade-off optimally privacy and utility.

\paragraph{Notation.} In this paper, $i$ will be used as an index for
an arbitrary \aRM{} in $[1, K]$, while $k$ will be used to indicate an optimal \aRM{} and $I_t$ is the \aRM{} played by an
\aGENT{} at \rOUND{} $t$. We use $g_{i,t}$ to indicate the \gAIN{}
of the $i$-th \aRM{} at \rOUND{} $t$. $\Regret_{\BASE{}}(T)$ is
the regret of the algorithm $\alg$ after $T$ \rOUNDS{}. The index and $T$ are
dropped when it is clear from the context. Unless otherwise specified, the regret is defined for oblivious adversaries against the fixed oracle. We use "$x \sim P$" to denote that $x$ is generated from distribution $P$. $\Laplace(\lambda)$ is used to denote the Laplace distribution with scale $\lambda$ while $\Bernoulli\left({p}\right)$ denotes the Bernoulli distribution with parameter $p$.

	\section{Algorithms and Analysis}

	\subsection{\DPALAP{}: Differential privacy through additional noise}
	\label{sec:noise}

	We start by showing that the obvious technique to achieve a given $\epsilon$-differential privacy in adversarial bandits already beat the state-of-the art. The main idea is to use any base bandit algorithm \BASE{} as input and add a Laplace noise of scale $\frac{1}{\epsilon}$ to each \gAIN{} before \BASE{} observes it. This technique gives $\epsilon$-DP differential privacy as the \gAINS{} are bounded in $[0, 1]$ and the noises are added i.i.d at each \rOUND{}.
	
	However, bandits algorithms require bounded \gAINS{} while the noisy \gAINS{} are not. The trick is to ignore \rOUNDS{} where the noisy
    \gAINS{} fall outside an interval of the form $[-b, b+1]$.  We pick the threshold $b$ such that, with high probability, the noisy \gAINS{} will be inside the interval $[-b, b+1]$.
    More precisely, $b$ can be chosen such that with high probability, the number of \rOUNDS{} ignored is lower than the upper bound $R_{\BASE}$ on the regret of \BASE{}. Given that in the standard bandit problem, the \gAINS{} are bounded in $[0,1 ]$, the \gAINS{} at accepted \rOUNDS{} are rescaled back to $[0, 1]$. 
    
    Theorem \ref{theo:dp:dp_a_laplace} shows that all these operations still preserve $\epsilon$-DP while Theorem \ref{theo:regret:dp_a_laplace} demonstrates that the upper bound on the expected regret of \DPALAP{} adds some small additional terms to $R_{\BASE}$. To illustrate how small those additional terms are, we instantiate \DPALAP{} with the \AEXP{} algorithm. This leads to a mechanism called \DPEXPLAP{} described in Algorithm \ref{alg:dpexplaplace}. With a carefully chosen threshold $b$, corollary \ref{cor:dpexplap} implies that the additional terms are such that the expected regret of \DPEXPLAP{} is $\O{\sqrt{T} \ln T/\epsilon}$ which is optimal in $T$ up to some logarithmic factors. This result is a significant improvement over the best known bound so far of $\O{T^{2/3} /\epsilon}$ from \cite{dplearningbanditandfull} and solves simultaneously the challenge (whether or not one can get $\epsilon$-DP mechanism with optimal regret) posed by the authors.
\begin{algorithm}                      
	\caption{\DPEXPLAP{}}                       
	\begin{algorithmic}
		\State Let $\tilde{G}_{i} = 0$ for all arms and $b = \frac{\ln T}{\epsilon}$, $\gamma = \sqrt{\frac{{K \ln K}}{(e-1)T}}$
		
		\ForAll{ round $t=1,\cdots, T$}
		
		\State Compute the probability distribution $p$ over the arms 
		\State $\qquad$ with $p = (p_{1,t}, \cdots p_{K,t})$ and $p_{i,t}$ as in eq \eqref{exp3:prob}.
		
		\State Draw an arm $I_t$ from the  probability distribution $p$.
		\State Receive the reward $g_{I_t,t}$
		
		\State Let the noisy gain be $g'_{I_t,t} = g_{I_t,t} + \Noise_{I_t,t}$ 
		 \State $\qquad$ with $\Noise_{I_t,t} \sim \Laplace(\frac{1}{\epsilon})$

		\If{$g'_{I_t,t} \in [-b, b+1]$}
			\State Scale $g'_{I_t,t}$ to $[0, 1]$
			\State Update the estimated cumulative gain of arm $I_t$: 
			\State $\qquad \tilde{G}_{I_t} = \tilde{G}_{I_t} + \frac{g'_{I_t,t}}{p_{I_t,t}}$
		\EndIf
		
		\EndFor
		
	\end{algorithmic}
	\label{alg:dpexplaplace}    
\end{algorithm}

			\begin{theorem}
				If \DPALAP{} is run with input a base bandit algorithm \BASE{}, the noisy reward $g'_{I_t,t}$ of the true reward $g_{I_t,t}$ set to $g'_{I_t,t} = g_{I_t,t} + \Noise_{I_t,t}$ with $ \Noise_{I_t,t} \sim \Laplace(\frac{1}{\epsilon})$, the acceptance interval set to $[-b, b+1]$ with the scaling of the rewards $g'_{I_t}$ outside $[0, 1]$ done using $g'_{I_t,t} = \frac{g'_{I_t,t} + b}{2b + 1}$;
			    then the regret $R_{\DPALAP}$ of \DPALAP{} satisfies:
				
				\begin{align}
					\E R_{\DPALAP}		\leq  \E R_{\BASE}^{scaled} + 2TK \exp(-\epsilon b) +  \frac{\sqrt{32T}}{\epsilon}
				\end{align}
				where $R_{\BASE}^{scaled}$ is the upper bound on the regret of \BASE{} when the rewards are scaled from $[-b, b+1]$ to $[0, 1]$
				\label{theo:regret:dp_a_laplace}
			\end{theorem}
			
			\begin{proof}[Proof Sketch]
				We observed that \DPALAP{} is an instance of \BASE{} run with the noisy rewards $g'$ instead of $g$. This means $R_{\BASE}^{scaled}$ is an upper bound of the regret $L$ on $g'$. Then, we derived a lower bound on $L$ showing how close it is to $R_{\DPALAP}$. This allows us to conclude.
			\end{proof}

			\begin{corollary}
				If \DPALAP{} is run with \EXPA{} as its base algorithm and $b = \frac{\ln T}{\epsilon}$, then its expected regret $\E R_{\DPEXPLAP}$ satisfies
				\begin{align*}
					\E R_{\DPEXPLAP}  &\leq \frac{4 \ln T}{\epsilon} \sqrt{(e-1)T K \ln K}\\
					  & \qquad + 2K +  \frac{\sqrt{32T}}{\epsilon}
				\end{align*}
				\label{cor:dpexplap}
			\end{corollary}
			\begin{proof}
				The proof comes by combining the regret of \AEXP{} \cite{obliviousauer03} with Theorem \ref{theo:regret:dp_a_laplace}
			\end{proof}

		\begin{theorem}
			\label{theo:dp:dp_a_laplace}
			\DPALAP{} is $\epsilon$-differentially private up to \rOUND{} $T$.
		\end{theorem}
		\begin{proof}[Proof Sketch]
			Combining the privacy of Laplace Mechanism with the parallel composition \cite{parallelcomposition} and post-processing theorems \cite{dpbook} concludes the proof.
		\end{proof}

		\subsection{Leveraging the inherent privacy of \EXPA{}}
		\label{sec:lever-inher-priv}

		\paragraph{On the differential privacy of \AEXP{}} 
		\cite{dpbook} shows that a variation of \EXPA{} for the full-information  setting (where the \aGENT{} observes the \gAIN{} of all \aRMS{} at any \rOUND{} regardless of what he played) is already differentially private. Their results imply that one can achieve the optimal regret with only a  sub-logarithmic privacy loss ($\O{\sqrt{128 \log T}}$) after $T$ \rOUNDS{}.
		
		We start this section by showing a similar result for \EXPA{} in Theorem \ref{theo:exp3-privacy}. Indeed, we show that  \EXPA{} is already differentially private but with a per-\rOUND{} privacy loss of $2$. \footnote{Assuming we want a sub-linear regret. See Theorem \ref{theo:exp3-privacy}} Our results imply that \EXPA{} can achieve the optimal regret albeit with a linear privacy loss of $\O{2T}$-DP after $T$ rounds. This is a huge gap compared with the full-information setting and underlines the significance of our result in section \ref{sec:noise} where we describe a concrete algorithm demonstrating that the optimal regret can be achieved with only a logarithmic privacy loss after $T$ rounds.

		\begin{theorem}
			\label{theo:exp3-privacy}
			The \AEXP{} algorithm is:  \[\min \left\{2T, T \cdot \ln \frac{K(1-\gamma) + \gamma}{\gamma}, 2(1-\gamma)T+ 2\sqrt{\frac{2 \ln T}{T}} \right\}\] differentially private up to \rOUND{} $T$.
			
			In practice, we also want \AEXP{} to have a sub-linear regret. This implies that $\gamma << 1$ and  \AEXP{} is simply $2T$-DP over $T$ \rOUNDS{}.
		\end{theorem}
		
		\begin{proof}[Proof Sketch]
		The first two terms in the theorem come from the observation that \EXPA{} is a combination of two mechanisms: the Exponential Mechanism \cite{mechanismdesigndp} and a randomized response. The last term comes from the observation that with probability $\gamma$ we enjoy a perfect $0$-DP. Then, we use Chernoff to bound with high probability the number of times we suffer a non-zero privacy loss.
		\end{proof}

       We will now show that the privacy of \EXPA{} itself may be improved without any additional noise, and with only a moderate impact on the regret.

	  \paragraph{On the privacy of a \AEXP{} wrapper algorithm}
	  The previous paragraph leads to the conclusion that it is impossible to obtain a sub-linear privacy loss with a sub-linear regret while using the original \AEXP{}. Here, we will prove that an existing technique is already achieving this goal. The algorithm which we called $\AEXP_{\tau}$ is from  \cite{policyregret12}. It groups the \rOUNDS{} into disjoint intervals of fixed size $\tau$ where the $j$'th interval starts on \rOUND{} $(j-1)\tau +1$ and ends on \rOUND{} $j\tau$. At the beginning of interval $j$, $\AEXP_{\tau}$ receives an action from \AEXP{} and plays it for $\tau$ \rOUNDS{}. During that time, \AEXP{} does not observe any feedback. At the end of the interval, $\AEXP_{\tau}$ feeds \AEXP{} with a single \gAIN{}, the average \gAIN{} received during the interval.
	  
	  Theorem \ref{theo:regret:exp3tau} borrowed from \cite{policyregret12} specifies the upper bound on the regret $\AEXP_{\tau}$. It is remarkable that this bound holds against the \emph{m-memory bounded adaptive adversary}. While in theorem \ref{theo:privacy:wrapper}, we show the privacy loss enjoyed by this algorithm, one gets a better intuition of how good those results are from corollary \ref{cor:exp:tau} and \ref{cor:exp:tau:inv}.
	   Indeed, we can observe that $\AEXP_{\tau}$ achieves a sub-logarithmic privacy loss of $\O{\sqrt{\ln T}}$ with a regret of $\O{T^{2/3}}$ against a special case of the \emph{m-memory bounded adaptive adversary} called the \emph{switching costs adversary} for which $m = 1$. This is the optimal regret bound (in the sense that there is a matching lower bound \cite{switchingcostlb}). This means that in some sense we are getting privacy for free against this adversary.
	   
	  \begin{theorem}[Regret of $\AEXP_{\tau}$ \cite{policyregret12}]
	  	The expected regret of $\AEXP_{\tau}$  is upper bounded by:	\[\sqrt{7T\tau K \ln K} + \frac{Tm}{\tau} + \tau\] against the \emph{m-memory bounded adaptive adversary} for any $m < \tau$.
	  	\label{theo:regret:exp3tau}
	  \end{theorem}
	  
	   \begin{theorem}[Privacy loss of $\AEXP_{\tau}$]	   	
	   	$\AEXP_{\tau}$ is $\left(\frac{4T}{\tau^3} + \sqrt{8 \ln(1/\delta') \frac{T}{\tau^3}}, \delta'\right)$-DP  up to \rOUND{} $T$.
	   	\label{theo:privacy:wrapper}
	   \end{theorem}   
	   \begin{proof}
	   	The sensitivity of each \gAIN{} is now $\frac{1}{\tau}$ as we are using the average. Combined with theorem \eqref{theo:exp3-privacy}, it means the per-round privacy loss is $2 \frac{T}{\tau}$. Given that \AEXP{} only observes $\frac{T}{\tau}$ \rOUNDS{}, using the advanced composition theorem \cite{dworkrv10} (Theorem III.3) concludes the final privacy loss over $T$ \rOUNDS{}. \end{proof}
	  
	  \begin{corollary}
	  	$\AEXP_{\tau}$ run with $\tau = (7K\log K)^{-1/3} T^{1/3}$ is $(\epsilon, \delta')$ differentially private up to \rOUND{} $T$ with $\delta' = T^{-2}$, $\epsilon = 28K\ln K + \sqrt{112K\ln K \ln T}$.
	  	Its expected regret against the \emph{switching costs adversary} is upper bounded by $ 2(7K \ln K)^{1/3} T^{2/3} +  (7K\log K)^{-1/3} T^{1/3}$.
	  	\label{cor:exp:tau}
	  \end{corollary}	  
	  \begin{proof}
	  	The proof is immediate by replacing $\tau$ and $\delta'$ in Theorem \ref{theo:regret:exp3tau} and \ref{theo:privacy:wrapper} and the fact that for the \emph{switching costs adversary}, $m = 1$.
	  \end{proof}

	  \begin{corollary}
	  $\AEXP_{\tau}$ run with $\tau = \left(\frac{4T\epsilon + 2T\ln \frac{1}{\delta}}{\epsilon^2}\right)^{1/3}$ is $(\epsilon, \delta)$ differentially private and its expected regret against the \emph{switching costs adversary} is upper bounded by: $\mathcal{O}\left({T^{2/3} \sqrt{K \ln K} \left(\frac{\sqrt{\ln \frac{1}{\delta}}}{\epsilon}\right)^{1/3}}\right)$
	  \label{cor:exp:tau:inv}
	  \end{corollary}

\section{Experiments}
\label{sec:experiments}
We tested \DPEXPLAP{}, $\EXPA{}_\tau$ together with the non-private \AEXP{} against a few different adversaries. The privacy parameter $\epsilon$ of \DPEXPLAP{} is set as defined in corollary \ref{cor:exp:tau}. This is done so that the regret of  \DPEXPLAP{} and $\EXPA{}_\tau$ are compared with the same privacy level. All the other parameters of \DPEXPLAP{} are taken as defined in corollary \ref{cor:dpexplap} while the parameters of  $\EXPA{}_\tau$ are taken as defined in corollary \ref{cor:exp:tau}. 

For all experiments, the horizon is $T=2^{18}$ and the number of \aRMS{} is $K=4$. 
We performed \NTRIALS{} independent trials and reported the \emph{median-of-means} estimator\footnote{Used heavily in the streaming literature~\cite{alon1996space}} of the cumulative regret. It partitions the trials into $a_0$ equal groups and return the median of the sample means of each group. Proposition \ref{prop:median-means} is a well known result (also in \cite{hsu2013loss,lerasle2011robust}) giving the accuracy of this estimator. Its convergence is $\O{\sigma/\sqrt{N}}$, with exponential probability tails, even though the random variable $x$ may have heavy-tails. In comparison, the empirical mean can not provide such guarantee for any $\sigma > 0$ and confidence in $[0, 1/(2e)]$ \cite{catoni2012challenging}.
\begin{proposition}
	\label{prop:median-means}
	Let $x$ be a random variable with mean $\mu$ and variance $\sigma^2 < \infty$. Assume that we have $N$ independent sample of $x$ and let $\hat{\mu}$ be the \emph{median-of-means} computed using  $a_0$ groups. With probability at least $1-e^{-a_0/4.5}$, $\hat{\mu}$ satisfies $\abs{\hat{\mu} -\mu} \leq \sigma \sqrt{6a_0/N}$.
\end{proposition} 
We set the number of groups to $a_0 = \NGROUPS{}$, so that the confidence interval holds w.p. at least $0.995$.

We also reported the deviation of each algorithm using the Gini's Mean Difference (GMD hereafter) \cite{gini1912variabilita}. GMD computes the deviation as $\sum_{j=1}^{N} (2j-N-1)x_{(j)}$ with $x_{(j)}$ the $j$-th order statistics of the sample (that is $x_{(1)} \leq x_{(2)} \leq \ldots \leq x_{(N)}$). As shown in \cite{yitzhaki2003gini,david1968miscellanea}, the GMD provides a superior approximation of the true deviation than the standard one. To account for the fact that the cumulative regret of our algorithms might not follow a symmetric distribution, we computed the GMD separately for the values above and below the \emph{median-of-means}.

At \rOUND{} $t$, we computed the cumulative regret against the fixed oracle 
who plays the best \aRM{} assuming that the end of the game is at $t$. The oracle uses the actual sequence of \gAINS{} to decide his best \aRM{}. For a given trial, we make sure that all algorithms are playing the same game by generating the \gAINS{} for all possible pair of \rOUND{}-\aRM{} before the game starts.

\paragraph{Deterministic adversary.}
As shown by \cite{audibert10a}, the expected regret of any \aGENT{} against an oblivious adversary can not be worse than that against the worst case deterministic adversary. In this experiment, \aRM{} $2$ is the best and gives $1$ for every even \rOUND{}.
 To trick the players into picking the wrong arms, the first arm always gives $0.38$ whereas the third gives $1$ for every \rOUND{} multiple of $3$. The remaining \aRMS{} always give $0$. As shown by the figure, this simple adversary is already powerful enough to make the algorithms attain their upper bound. 

\paragraph{Stochastic adversary}
\label{stochastic}
This adversary draws the \gAINS{} of the first \aRM{} i.i.d from $\Bernoulli \left({0.55}\right)$ whereas all other \gAINS{} are drawn i.i.d from $\Bernoulli \left({0.5}\right)$.

\paragraph{Fully oblivious adversary.}
\label{full_oblivious}
For the best arm $k$, it first draws a number $p$ uniformly in $[0.5, 0.5+2\cdot \varepsilon]$ and generates the gain $g_{k,t} \sim \Bernoulli \left({p}\right)$. For all other arms, $p$ is drawn from $[0.5-\varepsilon, 0.5+\varepsilon]$.  This process is repeated at every \rOUND{}. In our experiments, $\varepsilon = 0.05$

\paragraph{An oblivious adversary.}
This adversary is identical to the fully oblivious one  for every \rOUND{} multiple of $200$. Between two multiples of $200$ the last \gAIN{} of the arm is given.

\paragraph{The Switching costs adversary}
This adversary (defined at Figure 1 in  \cite{switchingcostlb}) defines a stochastic processes (including simple Gaussian random walk as special case) for generating the gains. It was used to prove that any algorithm against this adversary must incur a regret of $\O{T^{2/3}}$.

\paragraph{Discussion}
Figure~\ref{fig:experiments} shows our results against a variety of adversaries, with respect to a \emph{fixed oracle}.
Overall, the performance (in term of regret) of \DPEXPLAP{} is very competitive against that of \EXPA{} while providing a significant better privacy. This means that \DPEXPLAP{} allows us to get privacy for free in the bandit setting against an adversary not more powerful than the oblivious one.

The performance of $\EXPA{}_\tau$ is worse than that of \DPEXPLAP{} against an oblivious adversary or one less powerful. However, the situation is completely reversed against the more powerful switching cost adversary. In that setting, $\EXPA{}_\tau$ outperforms both \EXPA{} and  \DPEXPLAP{} confirming the theoretical analysis. We can see $\EXPA{}_\tau$ as the algorithm providing us privacy for free against switching cost adversary and adaptive m-bounded memory one in general.

\begin{figure*}
	\subfloat[Deterministic]{
	\includegraphics[width=0.45\textwidth]{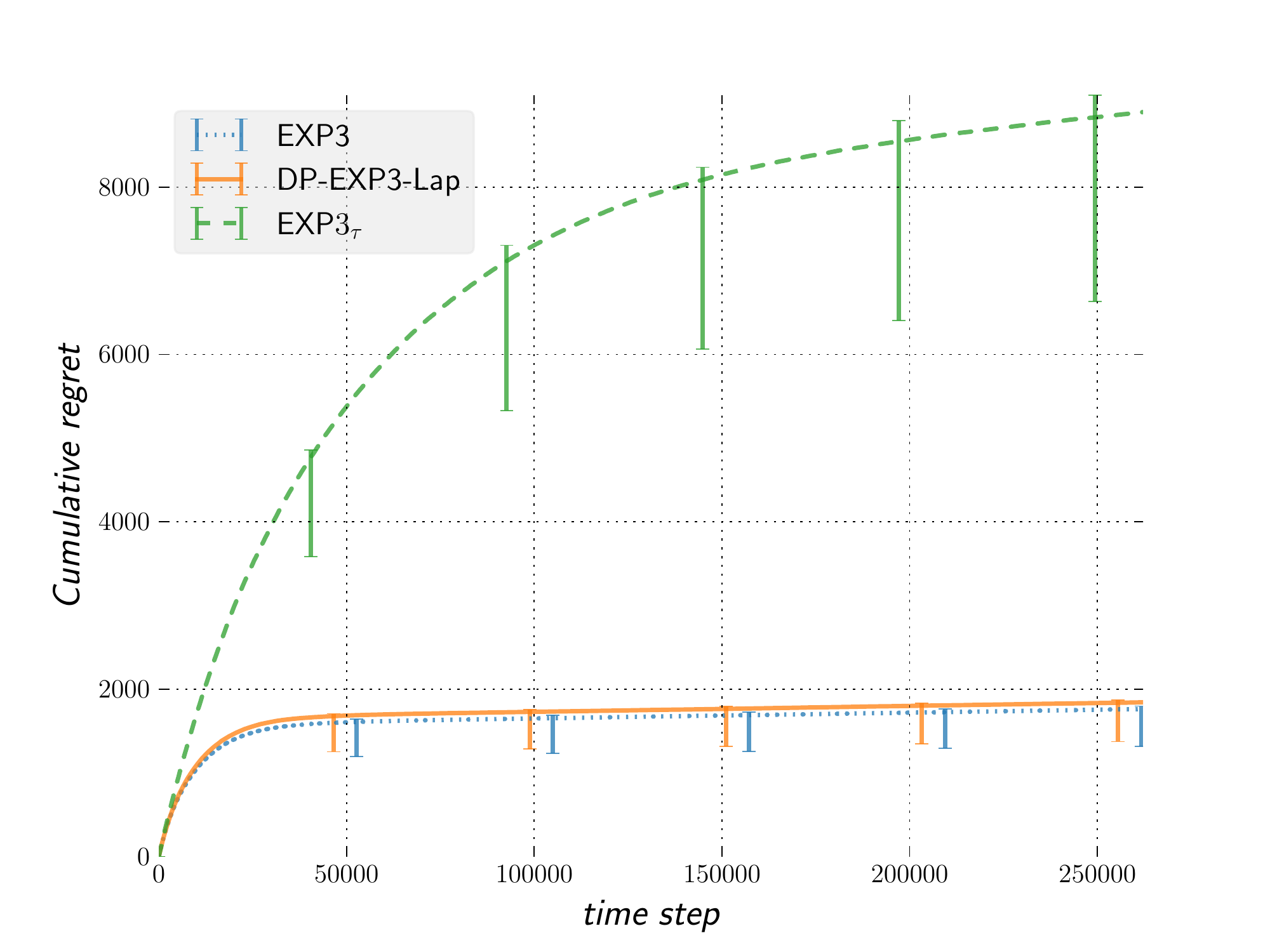}} \subfloat[Stochastic]{
		\includegraphics[width=0.5\textwidth]{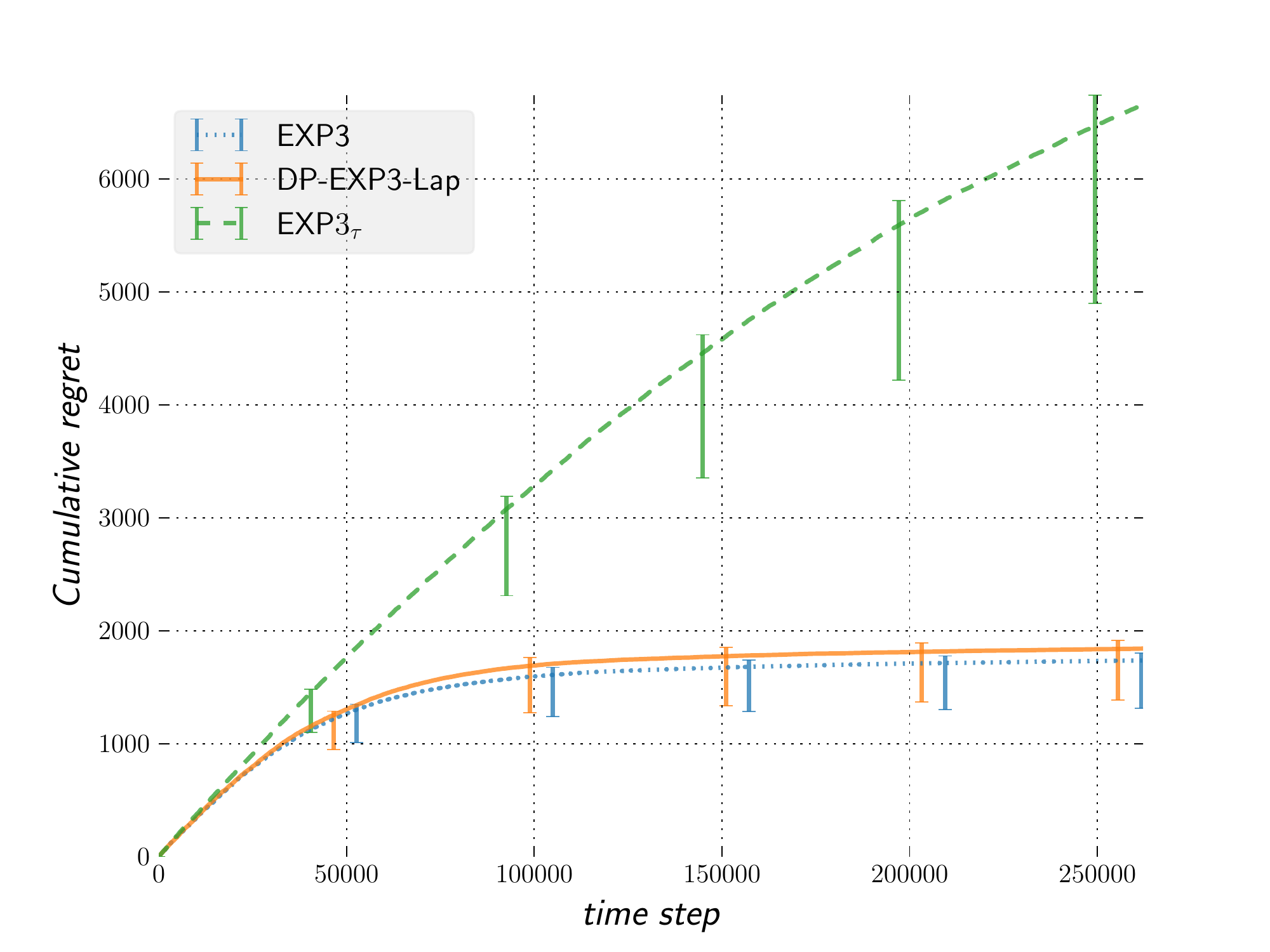}}\\ 
	\subfloat[Fully Oblivious]{
		\includegraphics[width=0.45\textwidth]{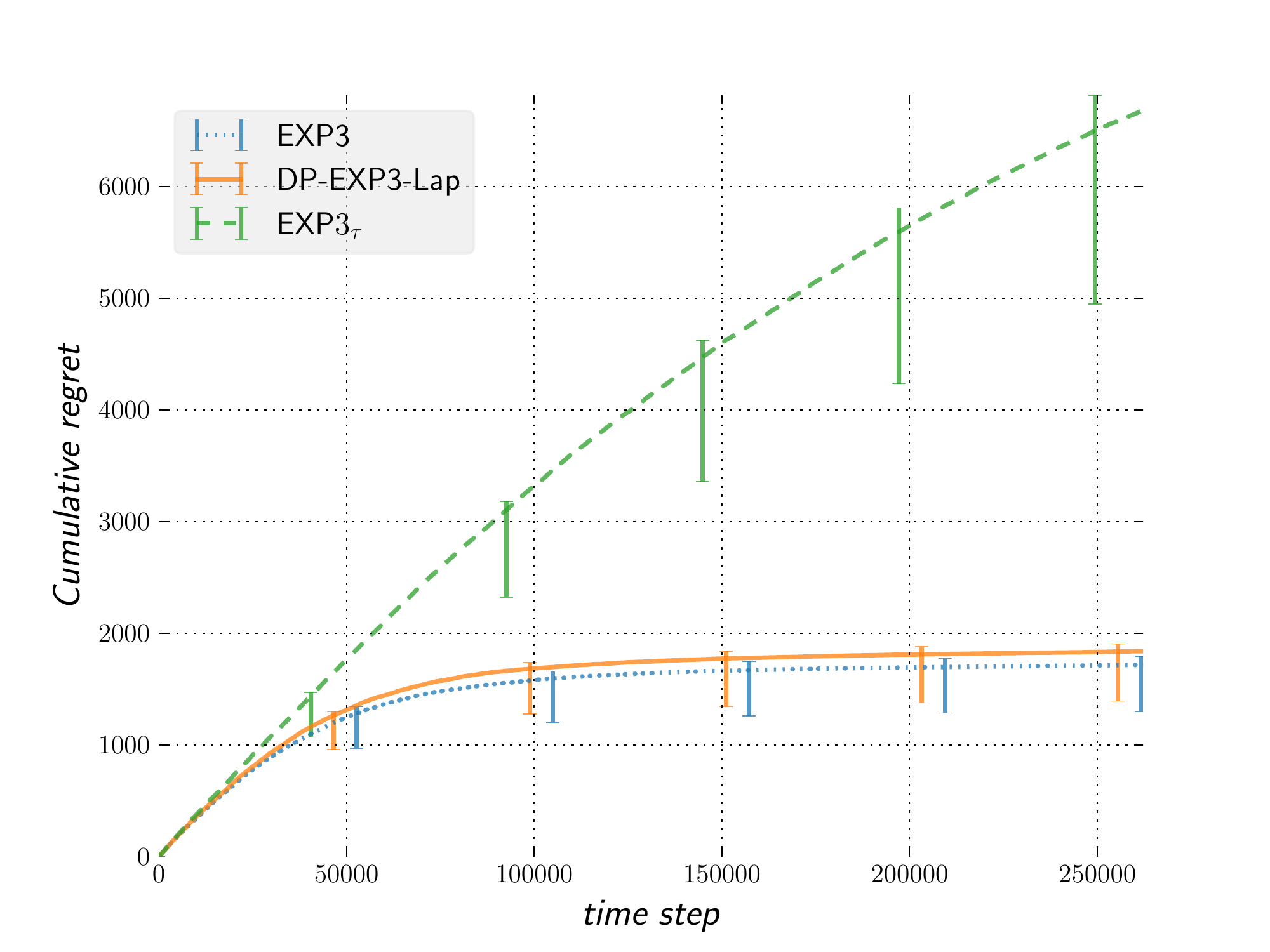}} \subfloat[Oblivious]{
		\includegraphics[width=0.5\textwidth]{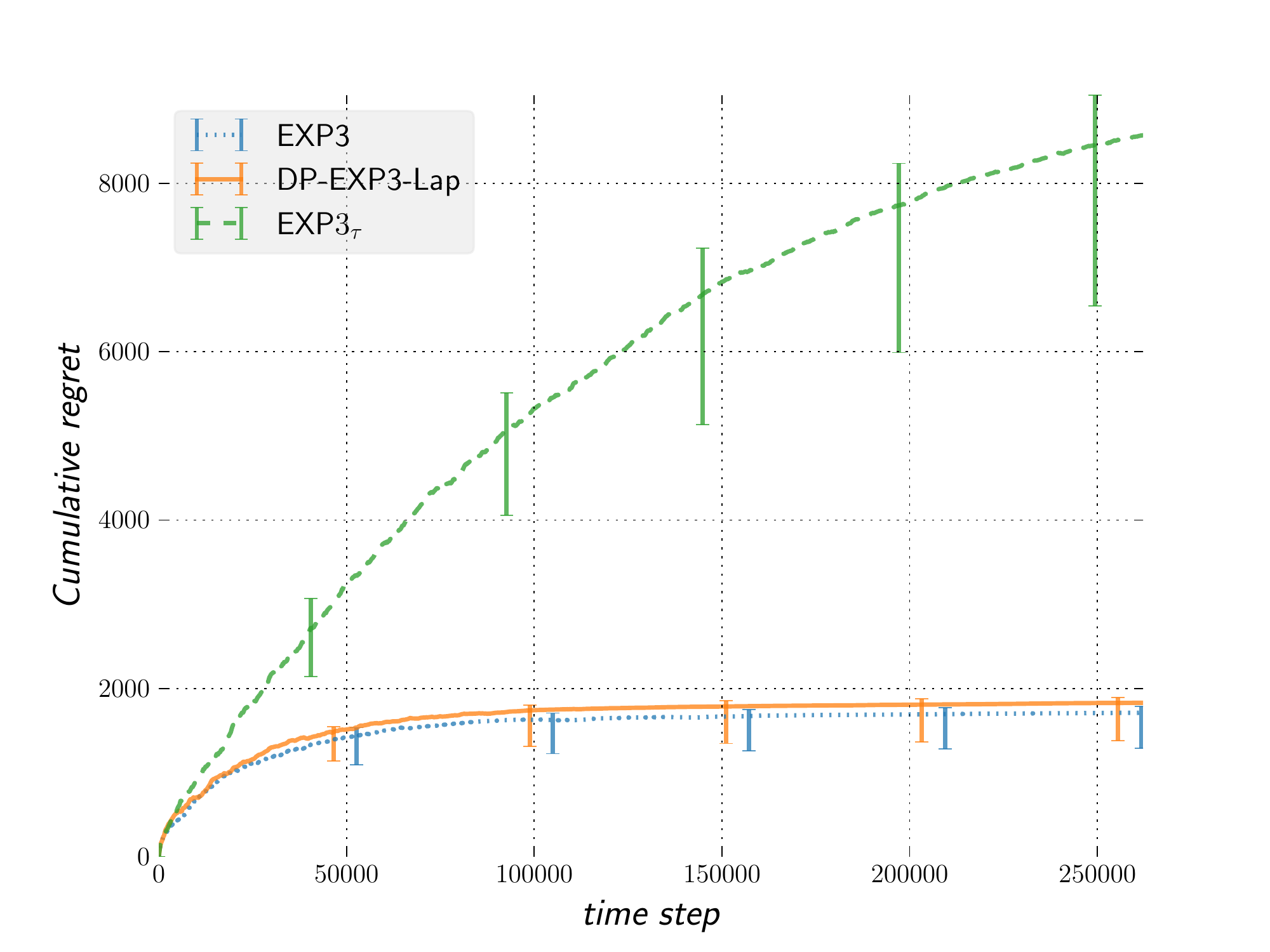}} \\
		\subfloat[Switching costs]{
		\includegraphics[width=0.5\textwidth]{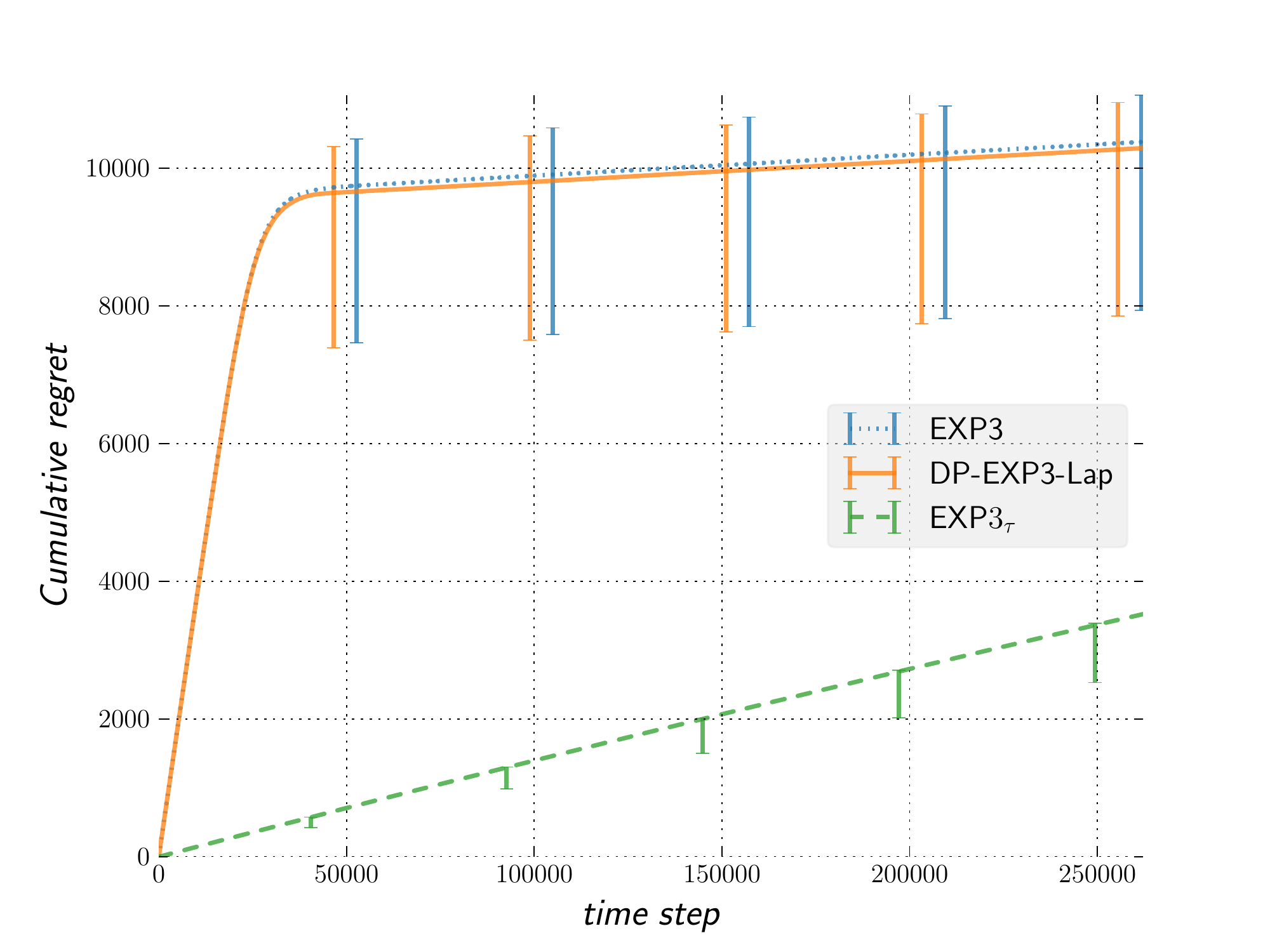}} 
	\caption{Regret and Error bar against five different adversaries, with respect to the fixed oracle}\label{fig:experiments}
	
\end{figure*}

\section{Conclusion}

We have provided the first results on differentially private
adversarial multi-armed bandits, which are optimal up to logarithmic
factors. One open question is how differential privacy affects regret
in the full reinforcement learning problem. At this point in time, the
only known results in the MDP setting obtain differentially private
algorithms for Monte Carlo policy
evaluation~\cite{Balle:DP-PE:ICML2016}. While this implies that it is
possible to obtain policy iteration algorithms, it is unclear how to
extend this to the full online reinforcement learning problem.

\paragraph{Acknowledgements.} This research was supported by the SNSF
grants ``Adaptive control with approximate Bayesian computation and
differential privacy'' and ``Swiss Sense Synergy'', by the Marie Curie
Actions (REA 608743), the Future of Life Institute 
``Mechanism Design for AI Architectures'' and the CNRS Specific Action
on Security.

\clearpage

\end{document}